\documentclass[11pt,reqno]{amsart}

\pdfoutput=1
\usepackage{amsfonts, amstext, amsmath, amsthm, amscd, amssymb, upgreek, mathtools, soul, float}
\usepackage[pdftex]{graphicx, color}
\usepackage[dvipsnames]{xcolor}
\usepackage[all,cmtip]{xy} 
\usepackage{tikz}
\usepackage{import}
\usepackage[hidelinks,pagebackref]{hyperref}
\usepackage{subfigure, wrapfig, overpic}
\usepackage{comment}
\usepackage[shortlabels]{enumitem}
\usepackage{etoolbox}
\usepackage{natbib}
\usepackage[right]{lineno}
\usepackage{array}
\usepackage[foot]{amsaddr}

\newcommand\restr[2]{{
  \left.\kern-\nulldelimiterspace
  #1
  \vphantom{\big|}
  \right|_{#2}
  }}

\makeatletter
\patchcmd{\@maketitle}
  {\ifx\@empty\@dedicatory}
  {\ifx\@empty\@date \else {\vskip3ex \centering\footnotesize\@date\par\vskip1ex}\fi
   \ifx\@empty\@dedicatory}
  {}{}
\patchcmd{\@maketitle}
  {\ifx\@empty\@date\else \@footnotetext{\@setdate}\fi}
  {}{}{}
\makeatother

\renewcommand*{\backref}[1]{}

\AtBeginDocument{
   \def\MR#1{}
}

\renewcommand{\leq}{\leqslant}
\renewcommand{\geq}{\geqslant}

\theoremstyle{plain}
\newtheorem{theorem}{Theorem}

\newtheorem{lemma}[theorem]{Lemma}

\newtheorem{proposition}[theorem]{Proposition}

\newtheorem*{namedtheorem}{\theoremname}
\newcommand{\theoremname}{testing}

\theoremstyle{definition}

\theoremstyle{remark}
\newtheorem{example}{Example}

\newcommand{\R}{\mathbb{R}}
\newcommand{\F}{\mathcal{F}}
\newcommand{\fix}{\mathrm{Fix}}
\newcommand{\id}{\mathrm{id}}
\newcommand{\crit}{\mathrm{Crit}}
\newcommand{\basin}{\mathrm{Basin}}
\newcommand{\qandq}{\quad \text{and} \quad}

\begin{document}

\title[Nowhere coexpanding functions]{Nowhere coexpanding functions}
\author[A.\ Cook]{Andrew Cook$^{\dagger,\ast,1}$}
\author[A.\ Hammerlindl]{Andy Hammerlindl$^{\ast,2}$}
\author[W.\ Tucker]{Warwick Tucker$^{\ast,3}$}

\address{$^{\dagger}$Corresponding author.}
\address{$^{\ast}$School of Mathematics, Monash University, Victoria 3800 Australia.}
\address{Email addresses:}
\address{$^1$\href{mailto:andrew.cook@monash.edu}{andrew.cook@monash.edu}, $^2$\href{mailto:andy.hammerlindl@monash.edu}{andy.hammerlindl@monash.edu}, $^3$\href{mailto:warwick.tucker@monash.edu}{warwick.tucker@monash.edu}.}

\date{14 September, 2023}

\begin{abstract}
We define a family of $C^1$ functions which we call ``nowhere coexpanding functions'' that is closed under composition and includes all $C^3$ functions with non-positive Schwarzian derivative. We establish results on the number and nature of the fixed points of these functions, including a generalisation of a classic result of Singer.
\end{abstract}

\maketitle

\section*{Lead paragraph}

\begin{bf}
The set of functions with negative Schwarzian derivative is closed under composition. As a result, it is a popular tool for studying the dynamics of functions in dimension one. A comprehensive treatment involving one-dimensional dynamics, with many results involving negative Schwarzian derivatives, is given in the book of \cite{1993melo}. More recent results appear in \cite{2000kozlovski,2009webb}, and \cite{2018mora}. In this paper, we study the dynamics of functions belonging to a new class $\F$. The class $\F$ includes all $C^3$ functions whose Schwarzian derivative is non-positive, but also includes many $C^1$ functions. Our main result states that for a function in $\F$ with no critical points, the set of fixed points is either an interval or has cardinality at most three. Further we adapt a classic result of \cite{MR494306} to apply to functions in $\F$, showing that for an attracting periodic point of a function in $\F$, either its immediate basin is unbounded, or its orbit attracts the orbit of a critical point. For some applications, functions in $\F$ can be glued together to make new functions in $\F$ which are in general not in $C^3$; we give examples of this in Section \ref{sec:gluing}.
\end{bf}

\section{Statement of main results}

To state our main results, we first introduce our function class $\F$ -- a subset of scalar $C^1$ functions. Consider a $C^1$ function $f : I \to \R$ defined on an interval $I \subset \R$. For now assume that $f$ has no critical points. If $f$ has two fixed points, $x$ and $y$ such that
\begin{equation} \label{equn:coexpanding}
f'(x)f'(y) > 1,
\end{equation}
then $x$ and $y$ are said to be {\bf coexpanding fixed points} of $f$.

We may generalise this definition to any two distinct points $x, y \in I$. Let $\Delta$ denote the diagonal of $I$, and define the function $\chi_f: I \times I \setminus \Delta \to \R$ as
\begin{equation*}
\chi_f(x, y) \coloneqq f'(x) f'(y) \left( \frac{x - y}{f(x) - f(y)} \right)^2.
\end{equation*}
\noindent
We say $x$ and $y$ (which are not necessarily fixed points) are {\bf coexpanding} if $\chi_f(x, y) > 1$. In other words, $x$ and $y$ are coexpanding if there is a scalar affine function $A:\R \to \R$ such that $x$ and $y$ are coexpanding fixed points of $A \circ f$. We say $f$ is {\bf nowhere coexpanding} if there are no such points in $I$.

Now suppose $f:\R \to \R$ and the set $\crit(f) \coloneqq \{x \in \R : f'(x) = 0\}$ of critical points is non-empty. We say $f$ is {\bf nowhere coexpanding} if its restriction to each connected component of $\R \setminus \crit(f)$ is nowhere coexpanding. Define $\F$ to be the set of all nowhere coexpanding functions. Observe that any affine function $A$ belongs to $\F$, since $\chi_A \equiv 1$ when $A$ is non-constant.

Section \ref{sec:coexpanding} shows that $\F$ is closed under composition, and classifies all possible sets of fixed points for functions in $\F$ with no critical points. Let $\fix(f) \coloneqq \{x \in \R: f(x) = x\}$. Our main result is as follows:

\begin{theorem} \label{thm:fix}
If $f \in \F$ has no critical points, then either $\fix(f)$ is a closed interval, or $f$ has at most three fixed points.
\end{theorem}

Section \ref{sec:schwarzian} provides a brief background on the Schwarzian derivative and shows that the class $\F$, when restricted to $C^3$, is the class of functions with non-positive Schwarzian derivative. This fact is made clear in the following result:

\begin{theorem} \label{thm:characterization}
Let $f:\R \to \R$ be a $C^3$ function. Then $f \in \F$ if and only if its Schwarzian derivative satisfies $S_f(x) \leq 0$ for all $x \in \R \setminus \crit(f)$.
\end{theorem}

In fact, Singer's result for functions with negative Schwarzian derivative may be generalised to nowhere coexpanding functions. To state the result, we need the following definitions. Let $p$ be an attracting period $n$ point of a function $f \in C^1$. The {\bf basin} of $p$ is
\[\basin(p) \coloneqq \{x \in \R : \lim_{k \to \infty} f^{n k}(x) = p\}.\]
The point $p$ is {\bf topologically attracting} if $\basin(p)$ is a neighbourhood of $p$. The {\bf immediate basin} of $p$ is the connected component of the basin containing $p$.

\begin{theorem} \label{thm:critical}
If $f \in \F$ and $p$ is a topologically attracting periodic point of $f$, then either the immediate basin of $p$ is unbounded, or there is a critical point of $f$ whose orbit is attracted to the orbit of $p$.
\end{theorem}

Section \ref{sec:coexpanding} proves this Theorem.

\section{Fixed points of nowhere coexpanding functions} \label{sec:coexpanding}

In this section, we build a better understanding of the properties of nowhere coexpanding functions. We first show that $\F$ is closed under composition. Then we prove a list of necessary conditions for the nowhere coexpanding property and use these conditions to prove Theorems \ref{thm:fix} and \ref{thm:critical}.

\begin{proposition} \label{prop:compositions}
The set $\F$ of nowhere coexpanding functions is closed under composition.
\end{proposition}

To show this, we need the following lemmas.

\begin{lemma} \label{lem:affine_composition}
For any non-constant affine functions $A$ and $B$, a function $f$ is nowhere coexpanding if and only if $A \circ f \circ B$ is nowhere coexpanding.
\end{lemma}

\begin{proof}
By the definition of a nowhere coexpanding function, $f \in \F$ if and only if $A \circ f \in \F$. Therefore, it suffices to show that a function $g$ is in $\F$ if and only if $h \coloneqq g \circ B$ is in $\F$. But these are equivalent, since a quick calculation shows that $\chi_{h}(x, y) = \chi_g(B(x), B(y))$ for all distinct $x$ and $y$.
\end{proof}

\begin{lemma} \label{lem:interval}
Consider $f, g \in \F$. If their composition $h \coloneqq g \circ f$ has no critical points on an interval $I$, then $\chi_h(x_1, x_2) \leq 1$ for any distinct $x_1, x_2 \in I$.
\end{lemma}

\begin{proof}
Let $A$ be the affine function such that $x_1$ and $x_2$ are fixed points of $\hat{h} \coloneqq A \circ h$. Let $y_1 = f(x_1)$, $y_2 = f(x_2)$, and let $B$ be the affine function taking $x_1$ to $y_1$, and $x_2$ to $y_2$. Now define $\hat{g} \coloneqq A \circ g \circ B$ and $\hat{f} \coloneqq B^{-1} \circ f$, so $\hat{h} = \hat{g} \circ \hat{f}$. Observe that $x_1$ and $x_2$ are fixed points of $\hat{f}$ and $\hat{g}$. By Lemma \ref{lem:affine_composition}, $\hat{f}, \hat{g} \in \mathcal{F}$, so
\[\hat{f}'(x_1) \hat{f}'(x_2) \leq 1 \qandq \hat{g}'(x_1) \hat{g}'(x_2) \leq 1,\]
and therefore
\[\hat{h}'(x_1) \hat{h}'(x_2) \leq 1.\qedhere\]

\end{proof}

\begin{proof}[Proof of Proposition \ref{prop:compositions}]
Apply Lemma \ref{lem:interval} to each connected component of $\R \setminus \crit(g \circ f)$.
\end{proof}

We now establish properties of the derivative of a nowhere coexpanding function.

\begin{lemma} \label{lem:loc_min}
If $f \in \F$ is regular on an interval $I$, then $f'$ has no local minima in $I$.
\end{lemma}

\begin{proof}
Suppose $f'$ has a local minimum at $p \in I$. Then there exist points $a$ and $b$ in $I$ with $a < p < b$ and such that $f'(a) = f'(b) > f'(x)$ for all $x$ in $(a, b)$. By the mean value theorem,
\[f'(a) > \frac{f(a) - f(b)}{a - b},\]
and therefore
\[f'(a) f'(b) \left(\frac{a - b}{f(a) - f(b)}\right)^2 > 1.\]
Thus $\chi_f(a, b) > 1$, showing that $f \notin \F$.
\end{proof}

\begin{lemma} \label{lem:contr_sandwich}
For a function $f \in \F$, if $\fix(f)$ has no interior and $a < b < c$ are fixed points such that $(a, c)$ contains no critical points, then $f'(b) > 1$.
\end{lemma}

\begin{proof}
Suppose $f'(b) \leq 1$. As $[a, b]$ and $[b, c]$ are invariant, $f'$ must average to unity on these intervals. Since $[a, b]$ and $[b, c]$ cannot be subsets of $\fix(f)$, there exist points $p \in (a, b)$ and $q \in (b, c)$ with $f'(p) > 1$ and $f'(q) > 1$, yielding a local minimum in $f'$ in $(p, q)$ which contradicts Lemma \ref{lem:loc_min}.
\end{proof}

We are now ready to prove Theorem \ref{thm:critical}.

\begin{proof}[Proof of Theorem \ref{thm:critical}]
Let $g \coloneqq f^n$, where $n$ is double the period of $p$. Thus, $g \in \F$ and $0 < g'(p) \leq 1$. Observe that replacing $f$ by $g$ does not change the immediate basin of $p$. If $x_c$ is a critical point for $g$, there is $0 \leq k < n$ such that $f^k(x_c)$ is a critical point for $f$. Further, if $x_c$ is attracted to the orbit of $p$ by $f$, then so is $f^k(x_c)$. Thus, it suffices to show that either the immediate basin of $p$ under $g$ is unbounded, or there is a critical point of $g$ whose orbit is attracted to $p$.

Let $I$ be the immediate basin of $p$. To prove by contradiction, assume there are no critical points in $I$ and that $I$ is bounded. Then both endpoints of $I$ are fixed points of $g$, contradicting Lemma \ref{lem:contr_sandwich}.
\end{proof}

For the remainder of this section, we narrow our focus to the set of nowhere coexpanding functions with no critical points, in order to prove Theorem \ref{thm:fix}.

\begin{lemma} \label{lem:discrete}
If $f \in \F$ has no critical points and $\fix(f)$ has no interior, then $f$ has at most three fixed points.
\end{lemma}

\begin{proof}
If $a$ and $b$ are fixed points, then by Lemma \ref{lem:contr_sandwich}, for any fixed point $p \in (a, b)$, $f'(p) > 1$. Since expanding fixed points cannot be adjacent, there is at most one fixed point in $(a, b)$. As $a$ and $b$ were arbitrarily chosen, $f$ cannot have more than three fixed points.
\end{proof}

\begin{lemma} \label{lem:one_eared_cat}
If $f \in \F$ has no critical points and $\fix(f)$ has non-empty interior, then $f'(x) \leq 1$ for all $x \in \R$.
\end{lemma}

\begin{proof}
As $\fix(f)$ has interior, there is an interval $[a, b] \subset \fix(f)$. By Lemma \ref{lem:affine_composition}, we can conjugate by a translation and reduce to the case where $a < 0 < b$. The result is immediate for $x \in [a, b]$. For $x \in [b, \infty)$,
\begin{equation} \label{equn:di}
\chi_f(0, x) = f'(0) f'(x) \left( \frac{x}{f(x)} \right)^2 \leq 1 \quad\text{and so}\quad f'(x) \leq \left( \frac{f(x)}{x} \right)^2.
\end{equation}
Observe that the solution $u(x) = x$ to the corresponding initial value problem,
\[u'(x) = \left( \frac{u(x)}{x} \right)^2, \quad u(b) = b\]
is unique. It is a standard result (see Chapter 3 of \cite{2002hartman} for instance) that since $f$ satisfies the differential inequality (\ref{equn:di}), it is majorised by $u$. That is, $f(x) \leq x$ for all $x \in [b, \infty)$ and so $f'(x) \leq 1$ follows from (\ref{equn:di}). A similar argument may be used for the case $x \in (-\infty, a]$.
\end{proof}

It is now easy to prove Theorem \ref{thm:fix}.

\begin{proof}[Proof of Theorem \ref{thm:fix}]
If $\fix(f)$ has non-empty interior, then by Lemma \ref{lem:one_eared_cat}, $f'(x) \leq 1$ for all $x \in \R$, and so $\fix(f)$ is connected. Otherwise, Lemma \ref{lem:discrete} applies.
\end{proof}

\section{The Schwarzian derivative} \label{sec:schwarzian}

In this section, we use Schwarzian derivatives to characterise $C^3$ functions in $\F$. Recall the {\bf Schwarzian derivative} of a $C^3$ function $f$ is defined by
\[S_f(x) = \frac{f'''(x)}{f'(x)} - \frac 3 2 \left(\frac{f''(x)}{f'(x)}\right)^2.\]

In complex analysis, the Schwarzian derivative measures how much a function differs from a M{\"o}bius transformation. In the realm of functions on $\mathbb{R}$, the Schwarzian derivative vanishes for affine functions, and is therefore a measure of how much a function varies from being affine. The following are standard results, which we will use in the proof of Theorem \ref{thm:characterization}. For details on Lemma \ref{lem:limit}, see for instance \cite{MR190326}.

\begin{lemma} \label{lem:limit}
Given a $C^3$ function $f$, let
\begin{equation*}
U_f(x, y) \coloneqq \frac{\partial^2}{\partial x \partial y} \log\left|\frac{f(x) - f(y)}{x - y}\right|,
\end{equation*}
for all distinct $x, y \in I$, where $I$ is a connected component of $\R \setminus \crit(f)$.
Then
\begin{equation*}
S_f(x) = 6 \lim_{y \to x} U_f(x, y).
\end{equation*}
\end{lemma}

The next lemma states the chain rule for Schwarzian derivatives as well as some properties that follow from it. The proof is omitted.

\begin{lemma} \label{lem:chain}
For $C^3$ functions $f$ and $g$,
\begin{enumerate}
\item For all $x \in \R \setminus \crit(g \circ f), \quad S_{g \circ f}(x) = S_g(f(x)) (f'(x))^2 + S_f(x).$
\item If $S_f(x) \leq 0$ for all $x \in \R \setminus \crit(f)$ and $S_g(y) \leq 0$, for all $y \in \R \setminus \crit(g)$, then $S_{g \circ f}(x) \leq 0$ for all $x \in \R \setminus \crit(g \circ f)$.
\item Suppose $g = B \circ f \circ A$, for non-constant affine functions $A$ and $B$. Then $S_f(x) \leq 0$ for all $x \in \R \setminus \crit(f)$ if and only if $S_g(y) \leq 0$ for all $y \in \R \setminus \crit(g)$.
\end{enumerate}
Properties (2) and (3) still hold if all the inequalities are changed to be strict.
\end{lemma}

The following lemma is adapted from Chapter 9.4 of \cite{MR3012659}.

\begin{lemma} \label{lem:s_positive}
Let $f$ be a $C^3$ function with fixed points $a < b$ such that no critical point of $f$ lies in $(a, b)$. If $f'(a) > 1$ and $f'(b) > 1$, then there exists $p \in (a, b)$ such that $S_f(p) > 0$.
\end{lemma}

\begin{proof}
Assume $S_f(x) \leq 0$ for all $x \in (a, b)$. Since $a$ and $b$ are fixed points, $f$ cannot be expanding for all $x$ in $(a, b)$. Thus, there exists $c \in (a, b)$ such that $f'(c) < 1$. Define $g(x) \coloneqq \frac{d}{dx} \log(f'(x)) = \frac{f''(x)}{f'(x)}$. By the mean value theorem applied to $\log(f'(x))$ on $[a, c]$, there exists $r \in (a, c)$ such that $g(r) < 0$. Similarly, there exists $t \in (c, b)$ such that $g(t) > 0$, and hence there exists $s \in (r, t)$ such that $g(s) = 0$. A simple calculation shows $g'(x) = S_f(x) + \frac 1 2 (g(x))^2$, and by our assumption, $g$ satisfies the differential inequality
\[g'(x) \leq \frac 1 2 (g(x))^2, \quad g(s) = 0.\]
Similar to the proof of Lemma \ref{lem:one_eared_cat}, it follows that $g(x) \leq 0$ for $x \in [s, b]$. This contradicts $g(t) > 0$, so our initial assumption was wrong.
\end{proof}

We are now ready to prove Theorem \ref{thm:characterization}.

\begin{proof}[Proof of Theorem \ref{thm:characterization}]
Suppose $f \in \F$. Evaluating $U_f$ in Lemma \ref{lem:limit} yields
\[(x - y)^2 U_f(x, y) = \chi_f(x, y) - 1.\]
Since $\chi_f(x, y) \leq 1$, we have $(x - y)^2 U_f(x, y) \leq 0$, so $U_f(x, y) \leq 0$. Therefore by Lemma \ref{lem:limit}, $S_f(x) \leq 0$ for all $x \in \R \setminus \crit(f)$.

Now suppose $f \notin \F$, so that $A \circ f$ has coexpanding fixed points $a < b$ for some affine function $A$ and the interval $(a, b)$ contains no critical point of $f$. Let $B$ be the affine function interchanging $a$ and $b$. Then $g \coloneqq A \circ f \circ B \circ A \circ f \circ B$ satisfies $g'(a) > 1$ and $g'(b) > 1$. By Lemma \ref{lem:s_positive}, there exists $p \in (a, b)$ such that $S_g(p) > 0$. Therefore, by items (2) and (3) of Lemma \ref{lem:chain}, $S_f(q) > 0$ for some $q$.
\end{proof}

\section{Examples} \label{sec:gluing}

For functions $f \in \F$ with no critical points, Theorem \ref{thm:fix} allows for zero, one, two, three, or infinitely many fixed points, where in the last case, $\fix(f)$ is some interval. In this section, we provide examples of functions if $\F$ for each of these cases. Further, we introduce a way of {\em gluing} some functions in $\F$ together to produce new functions in $\F$.

\begin{example} \label{eg:finite}
The following functions in $\F$ have no critical points and finitely many fixed points. In each case the number of fixed points is robust under $C^1$ small perturbations.

\begin{figure}[H]
\begin{centering}
\begin{tabular}{c c c c}
\includegraphics[width=0.2\linewidth]{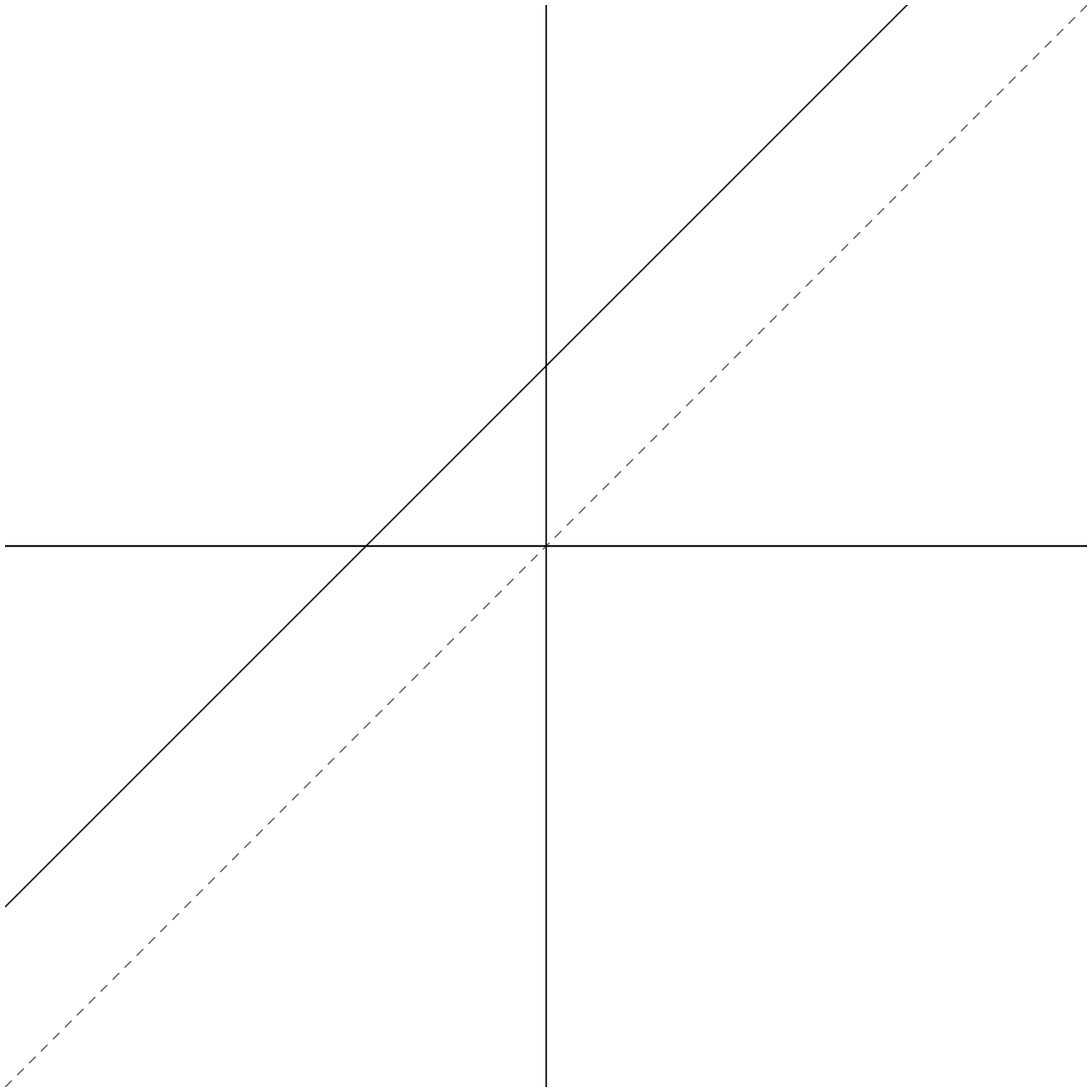}
&
\includegraphics[width=0.2\linewidth]{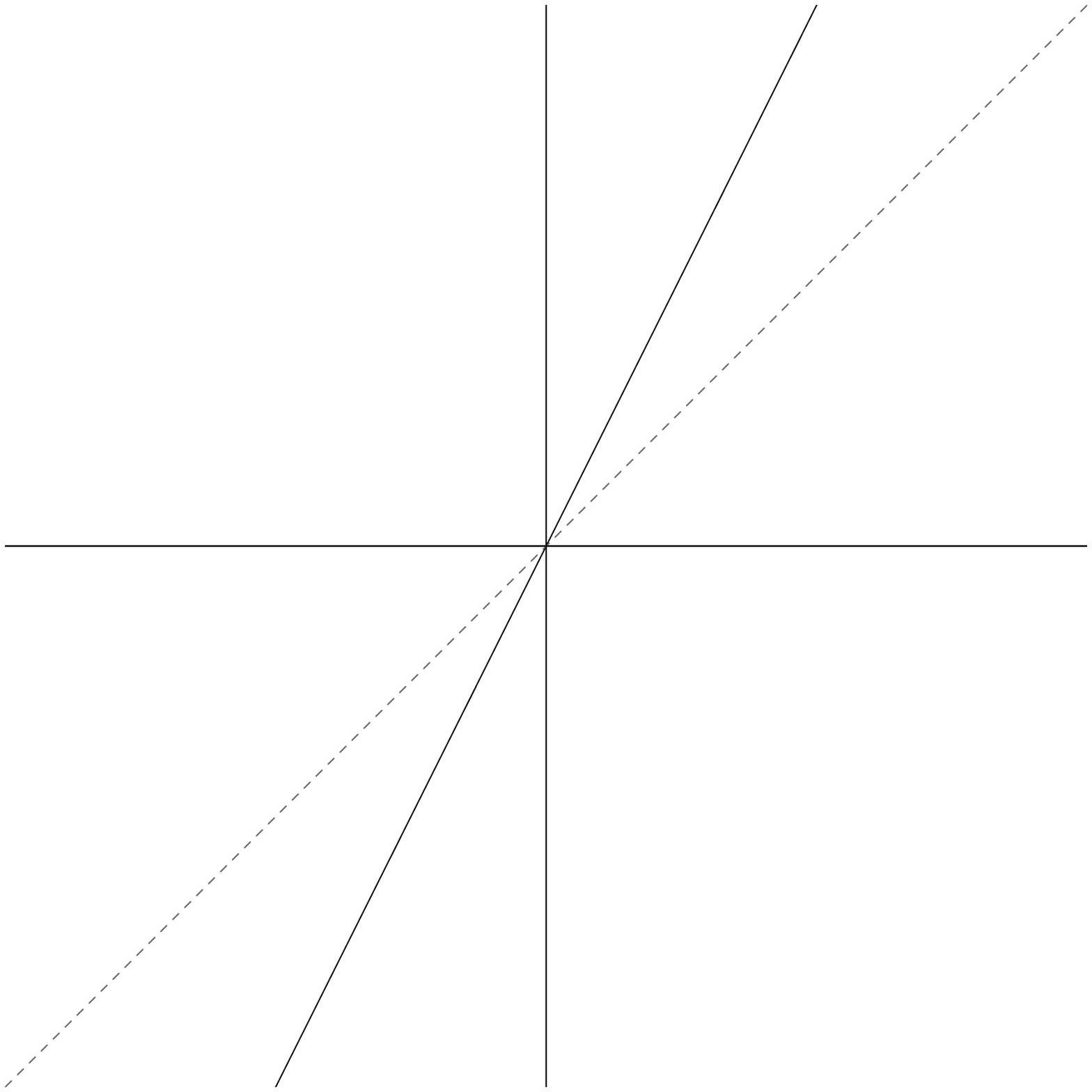}
&
\includegraphics[width=0.2\linewidth]{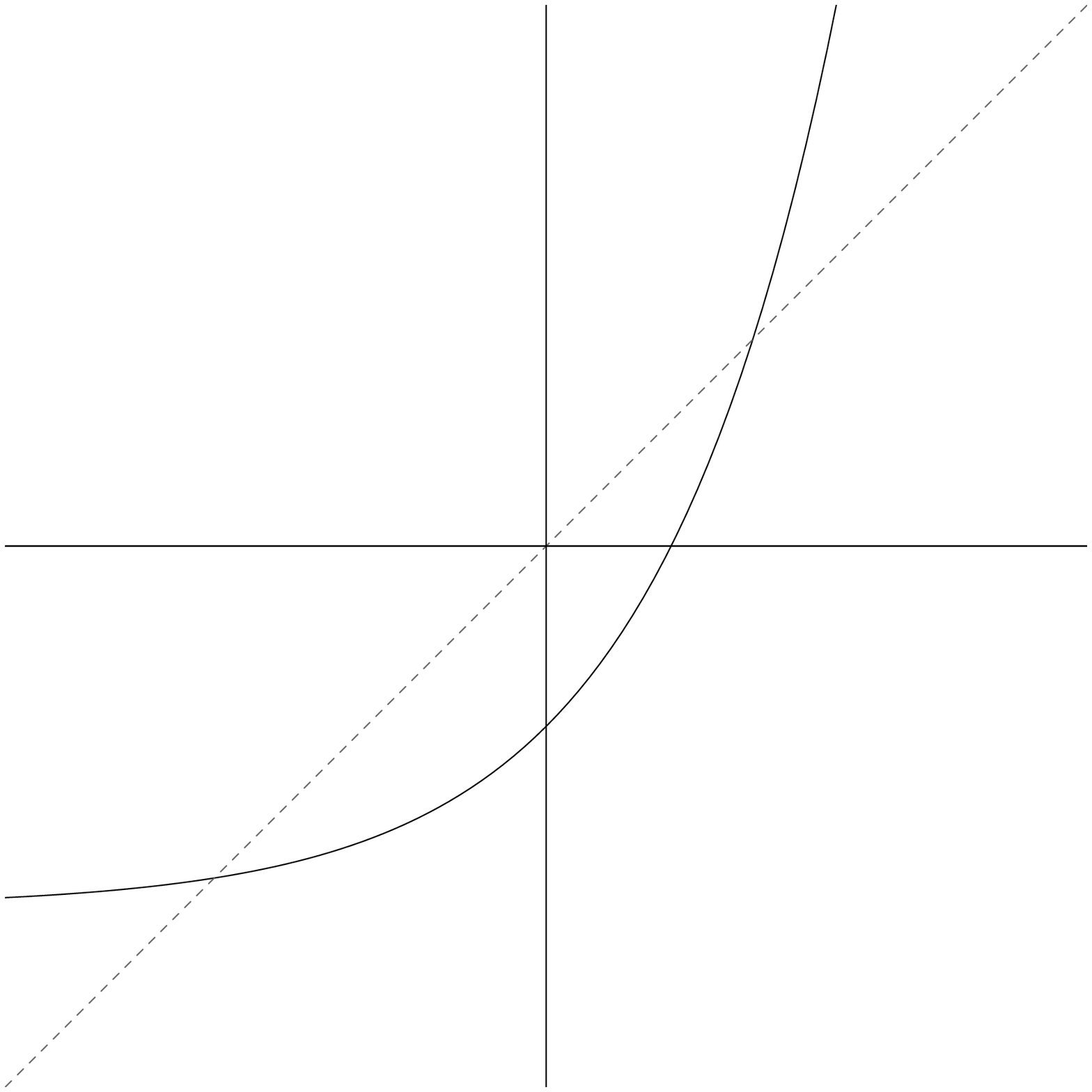}
&
\includegraphics[width=0.2\linewidth]{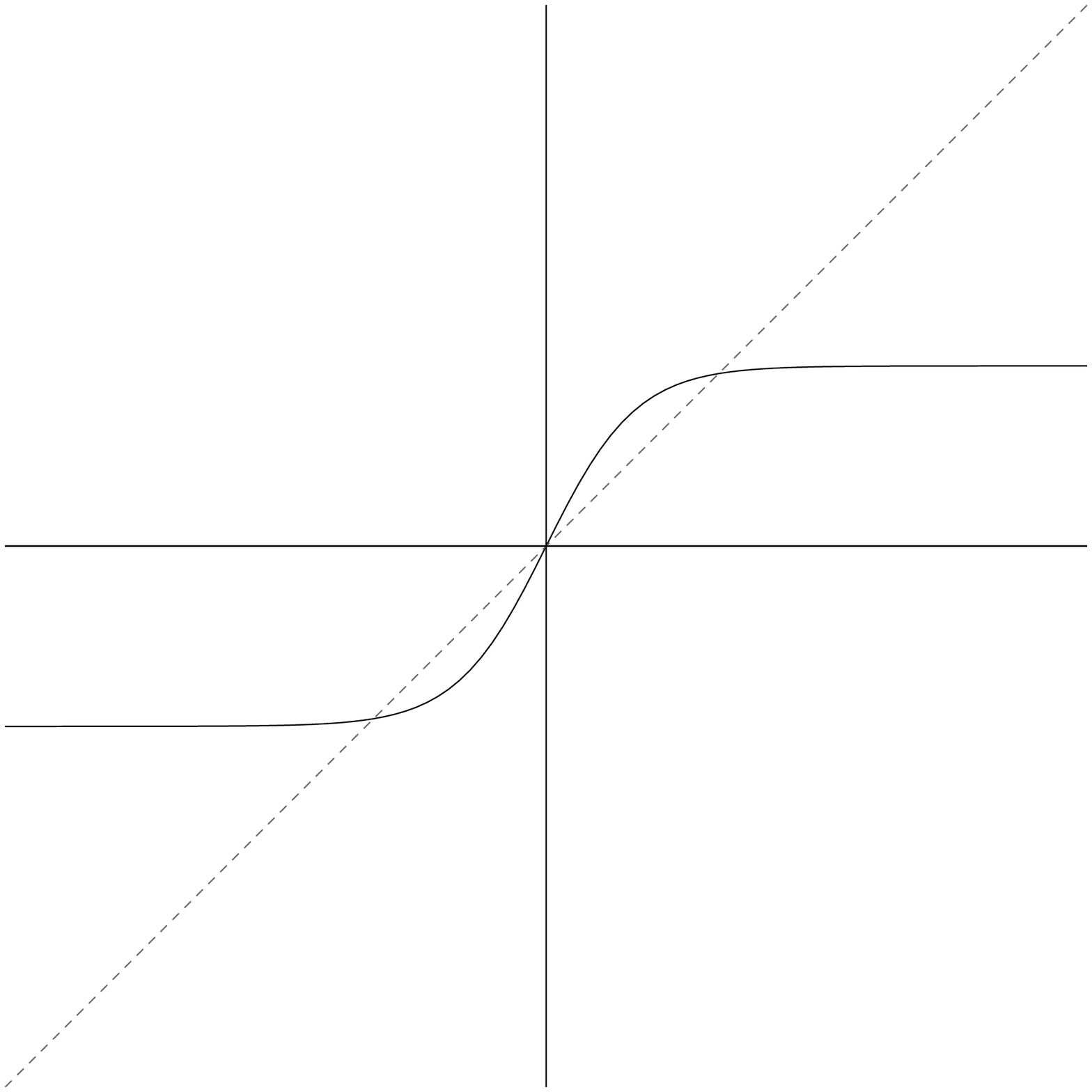}
\\
$f_0(x) = x + 1$ & $f_1(x) = 2 x$ & $f_2(x) = e^x - 2$ & $f_3(x) = \tanh(2 x)$\\
Zero fixed points & One fixed point & Two fixed points & Three fixed points\\
(a) & (b) & (c) & (d)
\end{tabular}
\end{centering}
\caption{Functions in $\F$ plotted with the diagonal in grey.} \label{fig:0to3}
\end{figure}

\noindent Theorem \ref{thm:characterization} shows that the above functions are in $\F$.
\end{example}

We now introduce another way of constructing new functions in $\F$ from other known functions in $\F$. We say a function $f \in \F$ is {\bf glueable} if $f'(0) = 1$ and $|f(x)| \leq |x|$ for all $x$ in the connected component of $\R \setminus \crit(f)$ containing zero. For glueable functions $f$ and $g$, define
\[(f \star g)(x) \coloneqq \begin{cases} 
    f(x), & x \leq 0\\
    g(x), & x \geq 0
  \end{cases}.
\]

\begin{lemma} \label{lem:glue}
If $f$ and $g$ are glueable, then $f \star g \in \F$.
\end{lemma}

\begin{proof}
By the definition of $f \star g$, we may assume without loss of generality that $f$ is odd. Take $x, y > 0$, where $-x$ and $y$ both lie in the connected component of $\R \setminus \crit(f \star g)$ containing zero. It suffices to show $\chi_{f \star g}(-x, y) \leq 1$. Since $f$ and $g$ are glueable, $x f(x) g(y) \leq x f(x) y$ and $y f(x) g(y) \leq y x g(y)$. Adding these together yields $(x + y) f(x) g(y) \leq x y (f(x) + g(y))$, and thus
\[\frac{f(x) g(y)}{f(x) + g(y)} \frac{x + y}{x y} \leq 1, \qandq \left(\frac{f(x)} x \right)^2 \left(\frac{g(y)} y \right)^2 \left(\frac{x + y}{f(x) + g(y)} \right)^2 \leq 1.\]
Since $f, g \in \F$,
\[\chi_f(0, y) = f'(0) f'(y) \left(\frac y {f(y)} \right)^2 \leq 1 \Rightarrow f'(y) \leq \left(\frac{f(y)} y \right)^2\]
and the same for $g$. Therefore
\[\chi_{f \star g}(-x, y) \leq f'(x) g'(y) \left(\frac{x + y}{f(x) + g(y)} \right)^2 \leq 1.\qedhere\]
\end{proof}

\begin{example}
Here, we show how to construct a function $h \in \F$ such that $\fix(h) = [a, b]$ for any prescribed, possibly unbounded, interval $[a, b]$.

Since $\F$ is closed under composition with affine functions, without loss of generality, we only need to find examples of $h$ where $\fix(h)$ is $\R$, $(-\infty, 0]$, and $[0, 1]$. We saw in Example \ref{eg:finite} that $\tanh \in \F$. Observe that $\tanh$ is also glueable. Furthermore, the identity function, $\id$, is the only glueable affine function. We treat each case of $\fix(h)$ separately.

\begin{enumerate}
\item Observe $h = \id \in \F$ is the only choice of $h$ satisfying $\fix(h) = \R$.
\item By Lemma \ref{lem:glue}, $\id \star \tanh \in \F$, so $h = \id \star \tanh$ is a suitable choice such that $\fix(h) = (-\infty, 0]$.
\item Define $g_a(x) \coloneqq (\id \star \tanh)(x - a) + a$ and choose $h = \tanh \star g_1$. Then $h \in \F$ and $\fix(h) = [0, 1]$.
\end{enumerate}
\end{example}

\section{Conclusions}

Part of our original motivation in studying this topic was to look at activation functions used in machine learning, the properties of these functions under composition, and of their fixed points. If an activation function lies in $\F$ then its composition with itself and with affine functions (in dimension one) will still lie in $\F$. This suggests that such a function might not be well suited to approximating general functions, but more research is needed to see what relevance this one-dimensional work has to high-dimensional neural networks.

Early in the development of machine learning, the logistic sigmoid function $1/(1+e^{-x})$ was a popular choice of activation function, and it lies in $\F$. More recently, the ReLU function has become popular. Since the ReLU function is not $C^1$, our theory here does not apply. However, there are more regular variants of the ReLU function. One of these is the ELU function \citep{2016clevert}, defined as gluing together the identity function $x \mapsto x$ for $x \geq 0$ and $x \mapsto e^x - 1$ for $x \leq 0$. Lemma \ref{lem:glue} shows that the ELU function lies in $\F$.

We conclude by demonstrating with an example that $\F$ is not closed under addition. Consider the function $f(x) \coloneqq \tanh(4 x) + \tanh(x/4)$, shown in Figure \ref{fig:tanh_counterexample}. Evaluating $S_{f}$ at $x = 1$ yields $S_{f}(1) > 1$, so by Theorem \ref{thm:characterization}, $f \notin \F$. Figure \ref{fig:composition_counterexample} shows how a composition of two copies of $f$ together with three affine functions yields a function with five fixed points. The affine functions are included to show how extra fixed points can be obtained.

A bounded $C^2$ function $f$ is {\bf sigmoidal} if $f'(x) > 0$ for all $x$, and $f$ has exactly one inflection point \citep{1995han}. There are many examples of sigmoidal functions in $\F$, such as $\tanh$, the logistic sigmoid, $\mathrm{erf}$, and $\arctan$. The function $f$ in the above example is also sigmoidal, but is not in $\F$.

\begin{figure}[H]
\begin{center}
\includegraphics[width=0.6\linewidth]{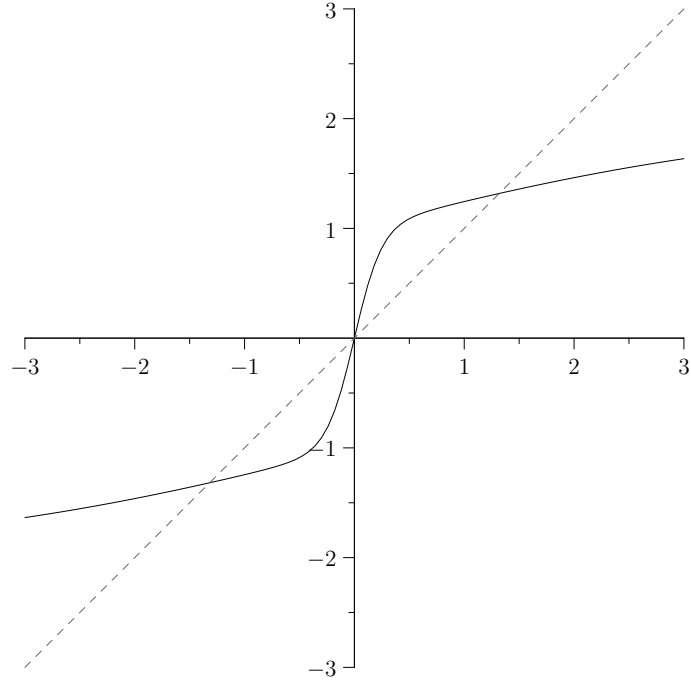}
\end{center}
\caption{The function $f(x) = \tanh(4 x) + \tanh(x/4)$.} \label{fig:tanh_counterexample}
\end{figure}

\begin{figure}[H]
\begin{center}
\includegraphics[width=0.57\linewidth]{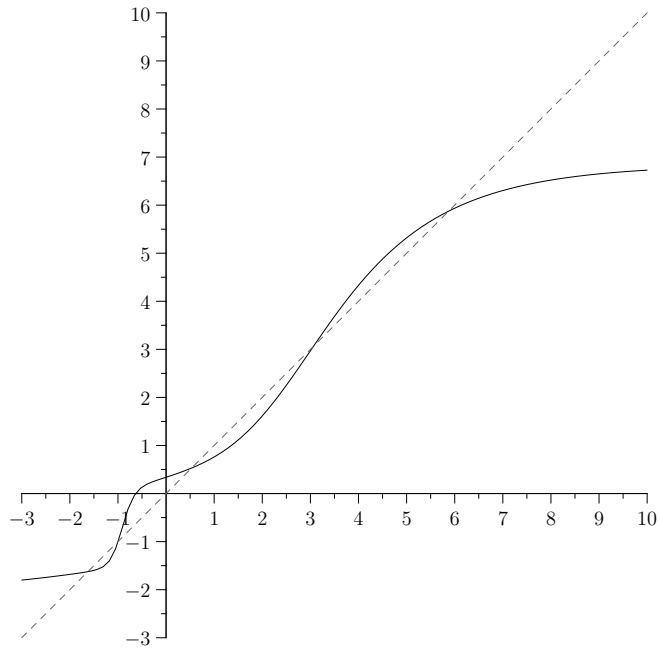}
\end{center}
\caption{A plot of the composition $x \mapsto 4 f(f(x + s) - 2 s) + s + 4$ with parameter $s = 0.94$ plotted with the diagonal to show where the five fixed points are.} \label{fig:composition_counterexample}
\end{figure}

\section{Acknowledgements}

This research is supported in part by an Australian Government Research Training Program (RTP) Scholarship. The authors would like to thank the anonymous reviewers for the helpful suggestions.

\bibliographystyle{abbrvnat}
\bibliography{references}

\end{document}